\documentclass[sigconf]{acmart}

\usepackage{booktabs} % For formal tables

\usepackage{amsthm} %new added %\usepackage{algpseudocode} %new added
\usepackage{comment} \usepackage{amsfonts} \usepackage{color}
\usepackage{algorithm} \usepackage{algorithmic} \usepackage{enumitem}
\usepackage{float} \usepackage{subcaption}
\usepackage{amsmath}
\usepackage{amssymb}
\usepackage{graphicx}
\usepackage{bm}
\usepackage{float}
\usepackage{leftidx}

\allowdisplaybreaks
\def\a{\mathbf{a}} %new added
\def\z{\mathbf{z}} %new added
\def\x{\mathbf{x}} %new added
\def\w{\mathbf{w}} %new added
\def\W{\mathbf{W}} %new added
\def\prob{\mathbb{P}} %new added
\def\D{\mathcal{D}} %new added
\def\S{\mathbf{S}} %new added
\def\bfD{\mathbf{D}} %new added
\def\bmd{\bm{\delta}} %new added
\def\mc{\mathcal} %new added
\newtheorem{Lemma}{Lemma} %new added
\newtheorem{Theorem}{Theorem}
\newtheorem{remark}{Remark}

\setcopyright{rightsretained}
\acmDOI{10.475/123_4}

\acmISBN{123-4567-24-567/08/06}

\begin{document}
\title{Differentially Private Generative Adversarial Network}

\author{Liyang Xie$^1$, Kaixiang Lin$^1$, Shu Wang$^2$, Fei Wang$^3$, Jiayu Zhou$^1$}
\affiliation{%
  \institution{$^1$Computer Science and Engineering, Michigan State University\\
               $^2$Department of Computer Science, Rutgers University\\
               $^3$Department of Healthcare Policy and Research, Weill Cornell Medical School
               }
}
\email{{xieliyan, linkaixi}@msu.edu, sw498@cs.rutgers.edu, few2001@med.cornell.edu, jiayuz@msu.edu}

\begin{abstract}
%!TEX root = main.tex
Generative Adversarial Network (GAN) and its variants have recently attracted
intensive research interests due to their elegant theoretical foundation and 
excellent empirical performance as generative models. These tools provide 
a promising direction in the studies where data availability is limited. 
One common issue in GANs is that the density of the learned generative distribution
could concentrate on the training data points, meaning that they can easily \emph{remember}
training samples due to the high model complexity of deep networks. This becomes a major concern 
when GANs are applied to private or sensitive data such as patient medical records, and the concentration of distribution may divulge critical
patient information. To address this issue, in this paper we propose a
differentially private GAN (DPGAN) model, in which we achieve differential privacy in GANs
by adding carefully designed noise to gradients during the learning
procedure. We provide rigorous proof for the privacy guarantee, as well as
comprehensive empirical evidence to support our analysis, where we demonstrate that our method can generate high quality data points at a
reasonable privacy level.
\end{abstract}





\maketitle
\section{Introduction}
%!TEX root = main.tex

% data sparsity issue. 
%Recent advances in data analytics have lead to many breakthroughs in solving 
%real-world problems, especially in areas with rich amount of data. 
In recent years, more and more data in different application domains are becoming readily available for the rapid development of both computer hardware and software technologies. 
Many data mining methodologies have been developed for analyzing those big data sets. One representative example is deep learning, which typically needs a huge amount of training samples to achieve promising performance. 
However, there exists domains where it is impossible to get as much data as we want. 
Medicine and Health Informatics are such fields. On individual patient level analysis, each patient is treated as a sample in model training process. However, considering the complexity of many diseases, the number of all patients from the whole world is still very small and far from enough. Moreover, we can never get the medical data from all patients for privacy and sensitivity reasons. Further, the expensive and time-consuming data collection process also limits the amount of data. Thus, the problem of building high-quality medical analytics models remains very challenging at present.

%significant challenges remain in areas with limited amount of data available. In the domain of medical informatics, for example, given sufficient data we are able to use existing machine learning algorithms to build high performance predictive models and thus greatly improve the healthcare quality by wisely allocating constrained medical resources. However, not only are there high costs in collecting and processing such medical data, privacy issues associated to the data have also largely limited the amount of data available for public research. Legislations like HIPAA~\cite{act1996health} has strict provisions for medical data usage in terms of privacy and security.  The data scarcity has imposed remarkable challenges in many learning  algorithms and substantially constrained the complexity of models to be used. 

% motivating generative model. 
Generative models~\cite{makhzani2015adversarial, rezende2014stochastic,
mescheder2017adversarial,burda2015importance,li2015generative} have provided
us a promising direction to alleviate the data scarcity issue. By sketching
the data distribution from a small set of training data, we are able to sample
from the distribution and generate much more samples for our study. By combining
the complexity of deep neural networks and game theory, the Generative
Adversarial Network (GAN)~\cite{goodfellow2014generative} and its variants have
demonstrated impressive performance in modeling the underlying data
distribution, generating high quality ``fake" samples that are hard to be
differentiated from real
ones~\cite{salimans2016improved,saito2016temporal,mogren2016c}. Ideally,
with the high quality generative distribution in hand, we can protect the
privacy of raw data by releasing only the distribution instead
of the raw data to the public or constrained individuals, and can even sample
datasets to fit our needs and conduct further analysis. 

% privacy issues associated to the generative models. 
However, the GANs can still implicitly disclose privacy information of the training
samples. The adversarial training procedure and the high model complexity of
deep neural networks, jointly encourage a distribution that is concentrated
around training samples. By repeated sampling from the distribution, there is
a considerable chance of recovering the training
samples~\cite{arjovsky2017wasserstein}. For example,
Hitaj~\emph{etal.}~\cite{hitaj2017deep} introduced an active inference attack
model that can reconstruct training samples from the generated ones.
Therefore, it is highly demanded to have generative models that not only
generates high quality samples but also protects the privacy of the
training data. 

% DPGAN (WGAN appears here only)
%To overcome the aforementioned privacy challenges in generative models, 
With the above considerations, in
this paper we propose a \emph{Differentially Private Generative Adversarial
Network} (DPGAN). DPGAN provides proven privacy control for the training data
from the sense of differential privacy~\cite{dwork2013algorithmic}.
Specifically, our proposed framework applies a combination of carefully
designed noise and gradient clipping, and uses the \emph{Wasserstein
distance}~\cite{arjovsky2017wasserstein} as an approximation of the distance
between probability distributions, which is a more reasonable metric than JS-
divergence in GAN. There are also prior works on studying differential privacy in deep learning models%Differentially private deep learning is studied
~\cite{abadi2016deep}. However, our DPGAN is different from~\cite{abadi2016deep} by clipping only on weights. 
We also proves that the
gradient can be bounded at same time, which avoids unnecessary distortion of
the gradient. This not only keeps the loss function with Lipschitz property but
also provides a sufficient privacy guarantee. Unlike the privacy preserving
deep framework mentioned in~\cite{papernot2017semi}, whose privacy loss is
proportional to the amount of data needed to be labeled in public data set,
the privacy loss of our DPGAN is irrelevant to the amount of generated
data. This makes our methods applicable under a wide variety of real world
scenarios. We evaluate DPGAN under various benchmark datasets
and network structures (fully connected networks and CNN), and demonstrate
that DPGAN can generate high-quality data points with sufficient
protection for differential privacy with reasonable privacy budget. 

The remaining of the paper is structured as follows: first, we will briefly overview the
related literature in Section~\ref{rw}, and then introduce the proposed DPGAN 
framework and theoretical properties in Section~\ref{a&t}. 
Our framework is evaluated in Section~\ref{exp} by the end.

\section{Related work}\label{rw}
%!TEX root = main.tex

In this section, we provide a brief literature review of relevant topics: generative adversarial
network, differential privacy and differentially private learning in neural
networks.

\noindent\textbf{Generative Adversarial Network.} 
GAN and its variants are developed in recent years with important advances
from the theoretical perspective. Instead of clipping the weights, Gulrajani
\emph{et al.}~\cite{gulrajani2017improved} improve the training stability and
performance of WGAN by penalizing the norm of the critical gradients with
respect to its input. Gulrajani \emph{et al.}~\cite{gulrajani2017improved} is aligned
with our differential privacy framework due to controlled value of gradient
norms. 

Zhao \emph{et al.}~\cite{zhao2016energy} introduces energy-based GAN (EBGAN),
which views the discriminator as an energy function that attributes low
energies to the regions near the data manifold and higher energies to other
regions. Similar to the original GANs, a generator is seen as being trained to
produce contrastive samples with minimal energies, while the discriminator is
trained to assign high energies to these generated samples. The  instantiation
of EBGAN framework use an auto-encoder architecture, with the energy being the
reconstruction error, in place of the usual discriminator.  The behavior of
EBGAN has shown to be more stable than regular GANs during training. Berthelot
\emph{et al.}~\cite{berthelot2017began} also use an autoencoder as a
discriminator and developed an equilibrium enforcing method, paired with a
loss derived from the Wasserstein distance. It improves over WGAN by balancing
the power of the discriminator and the generator so as to control the trade-
off between image diversity and visual quality. Qi~\cite{qi2017loss} proposes
a loss-sensitive GAN with Lipschitz assumptions on data distribution and loss
function. It improves WGAN by allowing the generator to focus on improving
poor data points that are far apart from real examples rather than wasting
efforts on those samples that have already been well generated, and thus
improving the overall quality of generated samples. Jones \emph{et
al.}~\cite{beaulieu2017privacy} used differentially private version of
Auxiliary Classifier GAN (AC-GAN) to simulate participants based on the
population of the SPRINT clinical trial. Choi \emph{et
al.}~\cite{choi2017generating} proposed medGAN, which is a generative
adversarial framework that can successfully generate EHR. However, the
approach  may have privacy concerns as we discussed earlier.

\noindent\textbf{Differential Privacy.} 
Differential privacy (DP)~\cite{dwork2006differential} and related algorithms have
been widely studied in the literatures. Examples include Dwork \emph{et
al.}~\cite{dwork2006calibrating} for sensitivity-based algorithm, which is among 
the most popular methods that protect privacy by adding noise to mask the
maximum change of data related functions. This work laid the theoretical
foundation of many DP studies. Chaudhuri {\it et al.}~\cite{chaudhuri2009privacy,
chaudhuri2011differentially} proposed DP empirical risk
minimization. The general idea of our DP framework has the
same spirits as the objective perturbation, which is different from adding
noise directly on the output parameters. Another related framework that adds
noise on gradient is Song \emph{et al.}~\cite{song2013stochastic}, which studied
DP variants of stochastic gradient descent. In their
empirical results, the practice of moderate increasing in the batch size can
significantly improve the performance. Song \emph{et al.}~\cite{song2015learning}
followed their early work~\cite{song2013stochastic}, and studied as how to use
stochastic gradient to learn from models trained by data from
multiple sources with DP requirements (hence multiple level of
noise). A comprehensive and structured overview of DP
data publishing and analysis can be found in~\cite{zhu2017differentially},
where several possible future directions and possible
applications are also mentioned.

\noindent\textbf{Differentially Private Learning in Neural Network.}  The
applications of DP in deep learning have been studied
recently in several literatures: Abadi \emph{et al.}~\cite{abadi2016deep} studied a
gradient clipping method that imposed privacy during the training procedure.
Shokri and Shmatikov~\cite{shokri2015privacy} for multi-party privacy
preserving neural network with a parallelized and asynchronous training
procedure. Papernot \emph{et al.}~\cite{papernot2017semi} combined Laplacian
mechanism with machine teaching framework. Phan \emph{et al.}~\cite{phan2017adaptive}
developed ``adaptive Laplace Mechanism'' that could be applied in a variety of
different deep neural networks while the privacy budget consumption is
independent of the number of training step. Phan \emph{et
al.}~\cite{phan2017preserving} developed a private convolutional deep belief
network by leveraging the functional mechanism to perturb the energy-based
objective functions of traditional CDBNs.

We propose DPGAN to address the challenges appeared in the previous works.
In~\cite{papernot2017semi} the privacy loss is proportional to the amount of
data labeled in that public data set, which may bring about unbearable privacy
loss. We solve this problem by training a differentially private generator and
can generate infinite number of data points without violating the privacy of
training data. Shokri and Shmatikov~\cite{shokri2015privacy} requires the
transmission of updated local parameters between server and local task,  which
is at risk of leakage of private information. Our framework addressed this
issue by avoiding a distributed framework. Also, our work is different
from~\cite{phan2017preserving} by adding noise within the training procedure
instead of adding noise on both energy functions and an extra softmax layer.

\section{Methodology}\label{a&t}
%!TEX root = main.tex

In this section, we elaborate the proposed privacy preserving framework DPGAN. 
Without loss of generality, we
discuss the DPGAN in the context of the WGAN framework~\cite{arjovsky2017wasserstein}
while we note that the proposed DPGAN technique can also be easily extended to other GAN
frameworks. We firstly introduce differential privacy and then conduct a brief
review of GAN and WGAN. We then introduce moments
accountant~\cite{abadi2016deep}, which is the key technique in our framework to set a
bound to the probability ratio so as to guarantee the privacy in the iterative gradient
descent procedure.

\subsection{Differential Privacy}

The privacy model used in our approach is differential
privacy~\cite{dwork2011differential}. Denote an algorithm with the
differential privacy property by $A_{p}(\cdot)$. The algorithm is randomized
in order to make it difficult for an observer to re-identify the input data,
where an observer is anyone who gets outputs of algorithms using the data.
Differential privacy (DP) is defined by~\cite{dwork2013algorithmic}:

\begin{definition}(Differential Privacy, DP)
\label{DP}
A randomized algorithm $\mc{A}_{p}$ is $(\epsilon, \delta)$-differentially private if for any two databases $\D$ and $\D'$ differing in a single point and for any subset of outputs $\S$:
    \begin{align}
    \prob(\mc{A}_{p}(\D)\in \S) \leq e^{\epsilon}\cdot \prob(\mc{A}_{p}(\D') \in \S) + \delta,
    \end{align}\label{def:dp}
\vspace{-0.2in}

\noindent where
$\mc{A}_{p}(\D)$ and $\mc{A}_{p}(\D')$ are the outputs of the algorithm
for input databases $\D$ and $\D'$, respectively, and $\prob$ is the
randomness of the noise in the algorithm. 
\end{definition}
\vspace{-0.05in}

It can be shown that the definition is equivalent to:
$$\left|\log
\left(\frac{P(\mc{A}_{p}(\D)=s)}{P(\mc{A}_{p}(\D')=s)}\right)\right|\leq \epsilon,$$ with
probability $1 - \delta$ for every point $s$ in the output range, where
$\epsilon$ reflects the \emph{privacy level}. A small $\epsilon$ ($\leq 1.0 $)
means that the difference of algorithm's output probabilities using $\D$ and
$\D'$ at $s$ is small, which indicates high perturbations of ground truth
outputs and hence high privacy, and vice versa. The non-private case is given
by $\epsilon=\infty$. $\delta$ measures the violation of the ``pure''
differential privacy. That is, there exists a small output range associated
with probability $\delta$ such that for some fixed point $s$ in this area, no
matter what the value of $\epsilon$ is, one can always find a pair of datasets
$\D$ and $\D'$ so that the inequality $|\log
(\frac{P(\mc{A}_{p}(\D)=s)}{P(\mc{A}_{p}(\D')=s)})|\geq \epsilon$ holds.
Typically we are interested in values of $\delta$ so that are less than the
inverse of any polynomial in the size of the database.

According to Def.~\ref{DP} and the intuition above, the noise
protects the membership of a data point in the dataset. For example, when
conducting a clinical experiment, sometimes a person does not want the
observer to know that he or she is involved in the experiment. This is due to
the fact that observer may link the test result to the
appearance/disappearance of certain person and harm the interest of that
person. A proper membership protection would ensure that replacing this person
with another one will not affect the result too much. This property holds only
if the algorithm itself is randomized, i.e. the output is associated with a
distribution. And this distribution will not change too much if certain data
point is perturbed or even removed. This exactly what the differential privacy
tries to achieve.

\subsection{GAN and WGAN} 
Generative adversarial nets~\cite{goodfellow2014generative} simultaneously
train two models: a generative model $G$ that transforms input distribution to
output distribution that approximates the data distribution, and a
discriminative model $D$ that estimates the probability that a sample came
from the training data rather than the output of $G$. Let $p_{\z}(\z)$ be the
input noise distribution of $G$ and $p_{data}(\x)$ be the real data
distribution, GAN aims at training $G$ and $D$ to play the following two-
player minimax game with value function $V(G,D)$:
\begin{align}\label{def:gangame}
 \min_{G}\max_{D} V(G,D) & = E_{\x\sim p_{data}(\x)}[log(D(\x))] \\
 &+ E_{\z\sim p_{\z}(\z)}[log(1-D(G(\z)))]. \nonumber
\end{align}
WGAN~\cite{arjovsky2017wasserstein} improves GAN by using the Wasserstein
distance instead of the Jensen---Shannon divergence. It solves a 
different two-player minimax game given by:
\begin{align}\label{def:wgangame}
 \min_{G}\max_{w \in W} E_{{\x}\sim p_{data}({\x})}[f_{w}(\x)]- E_{{\z}\sim p_{\z}(\z)}[f_{w}(G(\z))],
\end{align}
where functions $\{f_{w}(\x)\}_{w\in W}$ are all $K$-Lipschitz 
(with respect to $x$) for some $K$. 
Our approach exploits such $K$-Lipschitz property in WGAN and solves Formula~\ref{def:wgangame} in a differentially private manner.

\begin{algorithm}[b]
\small
\caption{Differentially Private Generative Adversarial Nets}
\label{alg:DPGAN}
\begin{algorithmic}[1]
\REQUIRE $\alpha_{d}$, learning rate of discriminator. $\alpha_{g}$, learning rate of generator. $c_{p}$, parameter clip constant. m, batch size. M, total number of training data points in each discriminator iteration. $n_{d}$, number of discriminator iterations per generator iteration. $n_{g}$, generator iteration. $\sigma_{n}$, noise scale. $c_{g}$, bound on the gradient of Wasserstein distance with respect to weights
\ENSURE Differentially private generator $\theta$.
\STATE Initialize discriminator parameters $w_{0}$, generator parameters $\theta_{0}$.
\FOR {$t_{1} = 1,\ldots ,n_{g}$}
    \FOR {$t_{2} = 1,\ldots ,n_{d}$}
        \STATE Sample $\{\z^{(i)}\}_{i=1}^{m}\sim p(\z)$ a batch of prior samples.
        \STATE Sample $\{\x^{(i)}\}_{i=1}^{m}\sim p_{data}(\x)$ a batch of real data points.
        \STATE For each $i$, $g_{w}(\x^{(i)},\z^{(i)}) \leftarrow \nabla_{w} \Big[f_{w}(\x^{(i)}) - f_{w}(g_{\theta}(\z^{(i)}))\Big] $
        \STATE $\bar{g}_{w} \leftarrow \frac{1}{m}(\sum_{i=1}^{m}g_{w}(\x^{(i)},\z^{(i)}) + N(0,\sigma_{n}^{2}c_{g}^{2}I))$. 
        \STATE $w^{(t_{2}+1)} \leftarrow w^{(t_{2})} + \alpha_{d} \cdot RMSProp(w^{(t_{2})},\bar{g}_{w})$
        \STATE $w^{(t_{2}+1)} \leftarrow clip(w^{(t_{2}+1)},-c_{p}, c_{p})$ 
    \ENDFOR
    \STATE Sample $\{\z^{(i)}\}_{i=1}^{m}\sim p(\z)$, another batch of prior samples.
    \STATE $g_{\theta} \leftarrow -\nabla_{\theta} \frac{1}{m}\sum_{i=1}^{m} f_{w}(g_{\theta}(\z^{(i)})) $
    \STATE $\theta^{(t_{1}+1)} \leftarrow \theta^{(t_{1})} - \alpha_{g} \cdot RMSProp(\theta^{(t_{1})},g_{\theta})$
\ENDFOR \\
\RETURN $\theta$.
\end{algorithmic}
\end{algorithm}

\subsection{DPGAN framework} 

Our method focuses on preserving the privacy during the training procedure
instead of adding noise on the final parameters directly, which usually
suffers from low utility. We add noise on the gradient of the Wasserstein
distance with respect to the training data. The parameters of discriminator
can be shown to guarantee differential privacy with respect to the sample
training points. We note that the privacy of data points that haven’t been
sampled for training is guaranteed naturally. This is because replacing these
data won't cause any change in output distribution, which is equivalent to the
case of $\epsilon = 0$ in Definition~\ref{def:dp}. The parameters of generator
can also guarantee differential privacy with respect to the training data.
This is becuase there is a post-processing property of differential
privacy~\cite{dwork2013algorithmic}, which says that any mapping (operation)
after a differentially private output will not invade the privacy. Here the
mapping is in fact the computation of parameters of generator and the output
is the differentially private parameter of discriminator. Since the parameters
of generator guarantee differential privacy of data, it is safe to generate
data after training procedure. In short, we have: differentially private
discriminator + computation of generator $\rightarrow$ differentially private
generator. This also means that even if the observer gets generator itself,
there is no way for him/her to invade the privacy of training data.

The DPGAN procedure is summarized in Algorithm~\ref{alg:DPGAN}. In line 9, the
clipping guarantees that $\{f_{w}(x)\}_{w\in W}$ are all $K_{w}$-Lipschitz
with respect to $x$ for some unknown $K_{w}$ and act in a way to bound the
gradient from each data point. The \textsc{RMSProp} in line 8 and line 13 is
an optimization algorithm that can adaptively adjust the learning rate
according to the magnitude of gradients~\cite{hinton2012neural}.

\subsection{Privacy Guarantees of DPGAN}
To show the DPGAN in Algorithm~\ref{alg:DPGAN} indeed protects the
differential privacy, we demonstrate that the parameters of generator $\theta$
(through discriminator parameters $w$) guarantee differential privacy with
respect to the sample training points. Hence any generated data from $G$ will
not disclose the privacy of training points. Through the \emph{moment
accountant} mechanism, we can compute the final composition result $\epsilon$.
By treating parameters of discriminator $w^{(t_{2}+1)}$ (line 9 in
Algorithm~\ref{alg:DPGAN}) as one point in the output space, it is easy to see
that the procedure of updating $w$ for fixed $t_{2}$ in any loop is just the
algorithm $\mc{A}_{p}$ in definition~\ref{def:dp}. Here the input of
$\mc{A}_{p}$ is  real data and noise and the output is the updated $w$. So we
have $\mc{A}_{p}(D) = M(aux,D)$ where $aux$ is an auxiliary input, which in
our algorithm refers to the previous parameters $w^{(t_{2})}$. Hence the
update of $w^{(t)}$ (line 3 to 10 in Algorithm~\ref{alg:DPGAN}) is an instance
of adaptive composition. Together with definition~\ref{DP}, it is natural to
define the following privacy loss at $o$:
\begin{definition}
\emph{(Privacy loss)}
    \begin{align*}
    c(o;M,aux,\D,\D') \triangleq \log \frac{\mathbb{P}[M(aux,D)=o]}{\mathbb{P}[M(aux,\D')=o]},
    \end{align*}
    \label{def:pl}
\vspace{-0.2in}
\end{definition}
\noindent which describes the difference between two distributions caused by changing data.
The privacy loss random variable is given by $C(M,aux,\D,\D') = c(M(\D);M,aux,\D,\D')$,
which is defined by evaluating the privacy loss at an outcome sampled from
$M(\D)$. Note that we assume the supports of 2 distributions associated
with $M(aux,D)$ and $M(aux,\D')$ are generally the same so it is safe to
evaluate them at same point $o$. This is a critical assumption since if there is an
area $s$ in support $M(aux,D)$ but not in $M(aux,\D')$, then evaluating
$C(M,aux,\D,\D')$ in $s$ will result in $\infty$ and violate the privacy.
We define the log of the moment generating function of the privacy loss
random variable and moments accountant as:
\begin{definition}
\emph{(Log moment generating function)}
    \begin{align*}
    \alpha_{M}(\lambda;aux,\D,\D') & \triangleq \log \mathbb{E}_{o\sim M(aux,D)} [exp(\lambda  
     C(M,aux,\D,\D'))].
    \end{align*}\label{def:logmgf}
\end{definition}
\begin{definition}
\emph{(Moments accountant)}
    \begin{align*}
    \alpha_{M}(\lambda) \triangleq \max_{aux,\D,\D'} \alpha_{M}(\lambda;aux,\D,\D').
    \end{align*}\label{def:ma}
\end{definition}
Moments accountant can be seen as the ``worst situation'' of the moment
generating function. The definition of moments accountant enjoys good
properties as mentioned in~\cite{abadi2016deep} (Theorem 2), where the
composability property shows that the overall moments accountant can be easily
bounded by the sum of moments accountant in each iteration, which brings about
a result that privacy is proportional to iterations. The tail bound can also
be applied in the privacy guarantee (Theorem 1 in same paper). We will use
this theorem to deduce our own result. Comparing with strong composition
theorem~\cite{dwork2010boosting}, moments accountant saves a factor of
$\sqrt{\log (n_{g}/\delta)}$. According to the definition~\ref{DP}, for a
large iteration $n_{g}$, this is a significant improvement.

In order to use moments accountant we need $g_{w}(\x^{(i)},\z^{(i)})$ to be
bounded (by clipping the norm in Algorithm 1 in~\cite{abadi2016deep}) and add
noise according to this bound. We do not clip the norm of
$g_{w}(\x^{(i)},\z^{(i)})$ , instead we show that by only clipping on $w$ can
we automatically guarantee a bound of the norm of $g_{w}(\x^{(i)},\z^{(i)})$.

\begin{lemma}
\label{lem:gw}
Under the condition of Alg.~\ref{alg:DPGAN}, assume that the activation
function of the discriminator has a bounded range and bounded derivatives
everywhere: $\sigma(\cdot) \leq B_{\sigma}$ and 
$\sigma^{'}(\cdot) \leq B_{\sigma^{'}}$, and every data point $\x$ satisfies $\|\x\| \leq B_{x}$,
then $\|g_{w}(\x^{(i)},\z^{(i)})\| \leq c_{g}$ for some constant $c_{g}$. 
\end{lemma}

\begin{proof}
    % Note that these assumptions are also used in literature like~\cite{qi2017loss}. 
    Without loss of generality, we assume $f_{w}$ is implemented using 
    a fully connected network. Let $H$ be the number of layers except input
    layer. Let $\W^{(l)}$ be the $l$-th weight matrix ($l=1,\ldots ,H$) whose 
    element $\W^{(l)}_{ij}$ is the weight connecting $j$-th node in layer 
    $l-1$ to $i$-th node in layer $l$. Let $\bfD^{(l)}$ be the diagonal 
    Jacobian of nonlinearities of $l$-th layer. We thus have:
    \begin{align}
        \bfD^{(l)}_{ij} = \left\{
            \begin{array}{lll}
            \sigma^{'}(\w^{(l)}_{i,:}\sigma(\z^{(l-1)})) & \textrm{if} & i = j  \\
            0 & \textrm{if} & i \neq j 
            \end{array}
            \right.
    \end{align}
    where $\w^{(l)}_{i,:}$ is the $i$th row of $W^{(l)}$ and 
    $\sigma(\z^{(l-1)})$ is the output of the $l-1$-th layer.
 The following fact is well known from the back-propagation algorithm on a fully connected network:
 \begin{align}
    \bmd^{(H)} 
       &= \nabla_{\a}C \odot \sigma^{'}(\z^{(H)}), \label{errL}\\
    \bmd^{(l)} 
       &= ((\W^{(l+1)})^{T}\bmd^{(l+1)}) \odot \sigma^{'}(\z^{(l)}), \label{errl}\\
    \frac{\partial C}{\partial \W_{jk}^{(l)}}   
       &= \a_{k}^{(l-1)}\bmd_{j}^{(l)}, \label{wgrad}
 \end{align}
where $C$ is the cost function, $\z^{(l)}$, $\a^{(l)}$ and $\bmd^{(l)}$ are the input, output and error vector of layer $l$, respectively. 
From~\ref{wgrad} we have for $l=2,\ldots ,H$:
\begin{align}\label{partialCW}
    \frac{\partial C}{\partial \W^{(l)}} &= \bmd^{(l)}(\a^{(l-1)})^{T} \notag \\
    &= (\bfD^{(l)}(\W^{(l+1)})^{T}\bmd^{(l+1)})(\a^{(l-1)})^{T} \notag \\
    &= (\bfD^{(l)}(\W^{(l+1)})^{T} \ldots \bfD^{(H-1)}(\W^{(H)})^{T}\bmd^{(H)}) \notag \\
    &* (\a^{(l-1)})^{T} \notag \\
    &= (\bfD^{(l)}(\W^{(l+1)})^{T} \ldots \bfD^{(H-1)}(\W^{(H)})^{T}) \notag \\
    &* (\a^{(l-1)})^{T}\sigma^{'}(z^{(H)}).
\end{align}
\vspace{-0.1in}
Take $\frac{\partial C}{\partial \W^{(l_{0})}}$ as an example:
\begin{align}
    [\bfD^{(l)}(\W^{(l+1)})^{T}]_{ij} &\leq c_{p}B_{\sigma^{'}} \notag \\
    [\bfD^{(l)}(\W^{(l+1)})^{T}\bfD^{(l+1)}(\W^{(l+2)})^{T}]_{ij} &\leq (c_{p}B_{\sigma^{'}})^{2}m_{l+1}, \notag 
\end{align} 
where we assume that $ c_{p} \leq \frac{1}{m_{l+1}B_{\sigma^{'}}}$. Here $m_{l+1}$ is the number of nodes in the $l+1$th layer.
And thus we have:   
\begin{align}
    \left[\prod_{l=l_{0}}^{H-1}\bfD^{(l)}(\W^{(l+1)})^{T}\right]_{ij} &\leq (c_{p}B_{\sigma^{'}})^{H-l_{0}}\prod_{l=l_{0}}^{H-2}m_{l+1}.
\end{align}
Because of the assumption that $\sigma(\cdot) \leq B_{\sigma}$, we have
$a^{(l-1)}_{j} \leq B_{\sigma}$. Combining it with~\ref{partialCW}, we have
$[\frac{\partial C}{\partial \W^{(l)}}]_{ij} \leq c_{p}B_{\sigma}
B_{\sigma^{'}}^{2}$ and therefore we have:
\begin{align*}
&\|g_{w}(\x^{(i)},\z^{(i)})\| 
 = \left\|\nabla_{w} \Big(f_{w}(\x^{(i)}) - f_{w}(g_{\theta}(\z^{(i)}))\Big)\right\| \\
 &\leq 2\left\|\nabla_{w} f_{w}(\x^{(i)})\right\| = 2\sum\nolimits_{l}\sum\nolimits_{ij}\left[\frac{\partial C}{\partial \W^{(l)}}\right]_{ij} \\
 &\leq 2c_{p}B_{\sigma} B_{\sigma^{'}}^{2}\sum\nolimits_{k=1}^{H-1}m_{k}m_{k+1} = c_{g},
\end{align*}
where the boundness of $g_{\theta}(\z^{(i)})$ comes from the choice of sigmoid
activation in the last layer of generator. Note that when computing $c_{g}$,
we need to take into consideration the dropout rate, weight sparsity, connection percentage of
convolutional nets, and other factors.
\end{proof}

\begin{remark}
Note that activation functions like \textsc{ReLU} (and its variants)  and
\textsc{Softplus} have unbounded $B_{\sigma}$.  This will not affect our
result because both the data points and weights are bounded, which guarantees
that the output of each node in each layer is bounded. The boundness of data
comes from a common fact that each data element has a bounded range.
\end{remark}

We have the following lemma which guarantees DP for discriminator training procedure.
\begin{Lemma}
\label{lemma:criticdp}
Given the sampling probability $q=\frac{m}{M}$, the number of discriminator iterations in
each inner loop $n_{d}$ and privacy violation $\delta$, for any positive
$\epsilon$, the parameters of discriminator guarantee $(\epsilon,
\delta)$-differential privacy with respect to all the data points used 
in that outer loop (fix $t_{1}$) if we choose:
    \begin{align}
    \sigma_{n} = 2q\sqrt{n_{d}\log(\frac{1}{\delta})} \bigg/ \epsilon.
    \end{align}
\end{Lemma}
\begin{proof}
The DP guarantee for the discriminator training procedure follows
from the intermediate result~\cite{abadi2016deep} (Theorem 1). We need to
find an explicit relation between $\sigma_{n}$ and $\epsilon$, i.e., how much
noise standard deviation $\sigma_{n}$ we need to impose on the gradient so
that we can guarantee a privacy level $\epsilon$, with small violation $\delta$. Combine inequality
$n_{d}q^{2}\lambda^{2}/\sigma^{2} \leq \lambda\epsilon/2$ and inequality
$e^{-\lambda\epsilon/2} \leq \delta$ in Theorem 1, we can get the result by
letting the equality hold.
\end{proof}
Lemma~\ref{lemma:criticdp} quantifies the relation between noise level
$\sigma_{n}$ and privacy level $\epsilon$. It shows that for fixed
perturbation $\sigma_{n}$ on gradient, larger $q$ leads to less privacy
guarantee (larger $q$). This is indeed true  since when more data are involved
in computing discriminator $w$, less privacy is assigned on each of them.
Also, more iterations ($n_{d}$) leads to less privacy because the observer gives
more information (specifically, more accurate gradient) for data. This
requires us to choose the parameters carefully in order to have a reasonable
privacy level.
% The privacy of each data here is just $\epsilon$
%
Finally we have the following theorem as the privacy guarantee of the 
parameters of the generator: 
\begin{Theorem}
\label{thm:dpgan}
The output of generator learned in Algorithm~\ref{alg:DPGAN} guarantees $(\epsilon, \delta)$-differential privacy.
\end{Theorem}
% \begin{proof}
The privacy guarantee a direct consequence from Lemma~\ref{lemma:criticdp} followed by the post-processing property of differential privacy~\cite{dwork2013algorithmic}.
% \end{proof}

\section{Experiment}\label{exp}
%!TEX root = main.tex 

\begin{figure*}[t!]\small
\centering
\begin{tabular}{c | c | c | c}
\includegraphics[scale=0.19]{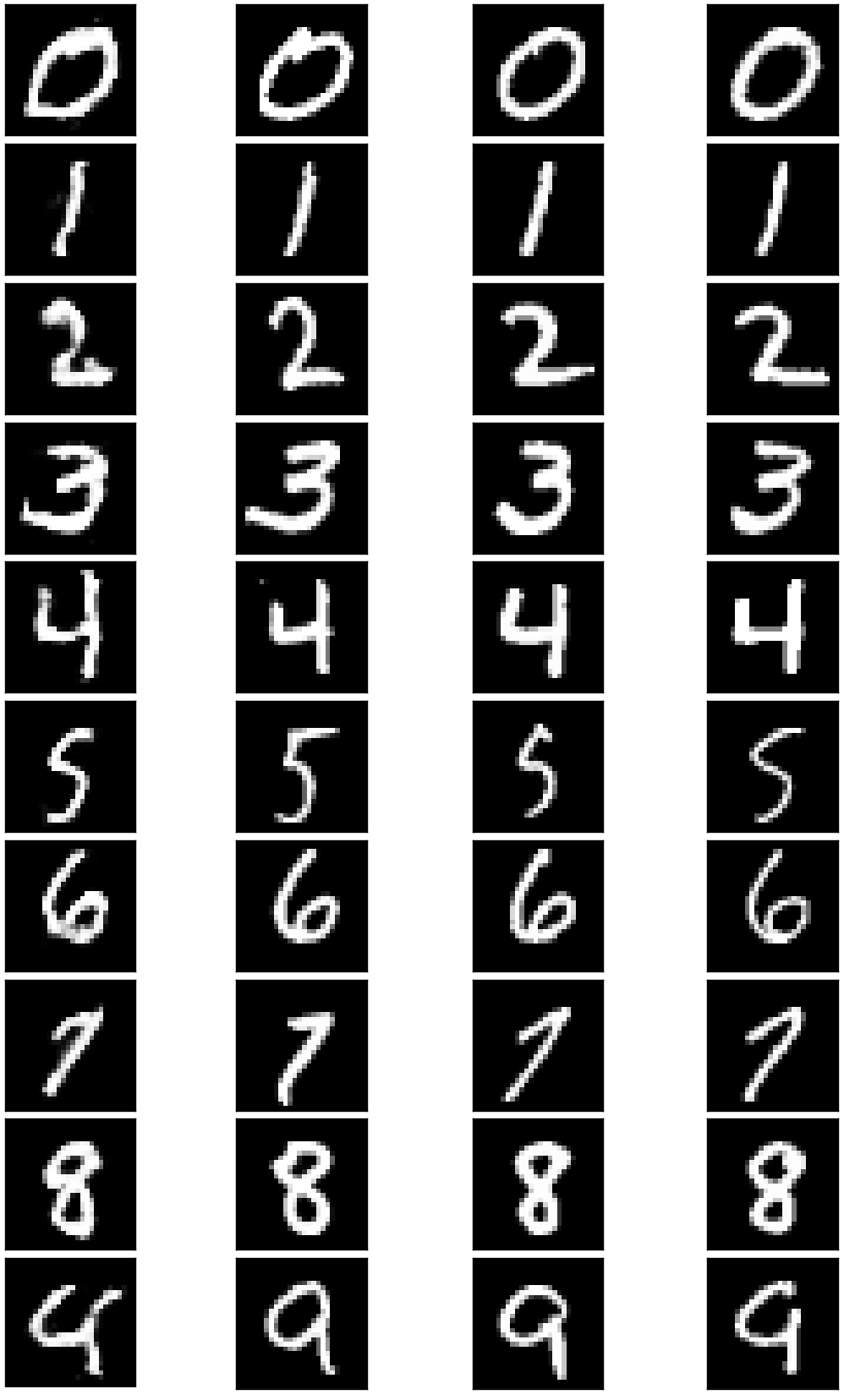} \hspace{+1in}& 
\includegraphics[scale=0.19]{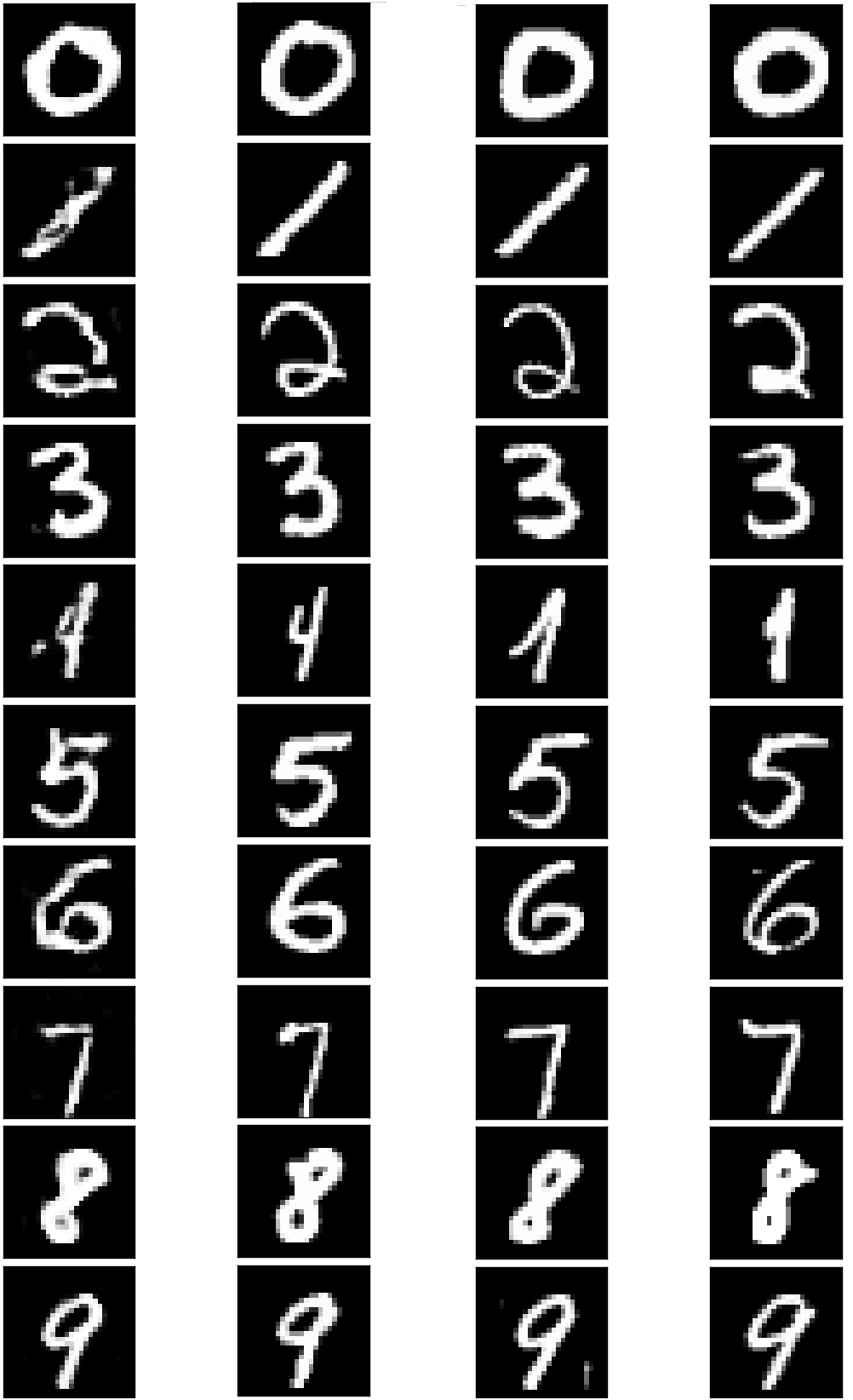} \hspace{+1in}& 
\includegraphics[scale=0.19]{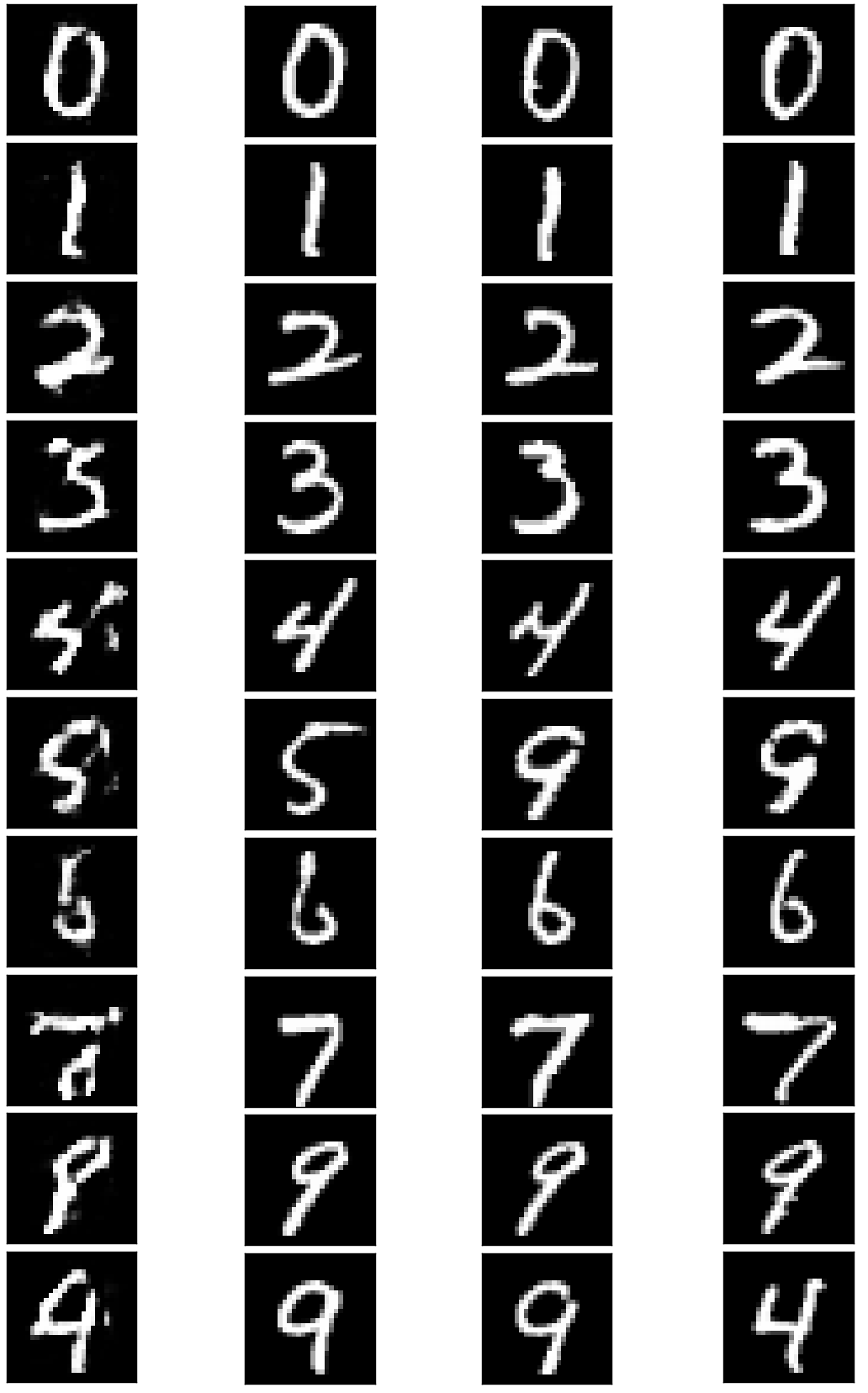} \hspace{+1in}& 
\includegraphics[scale=0.19]{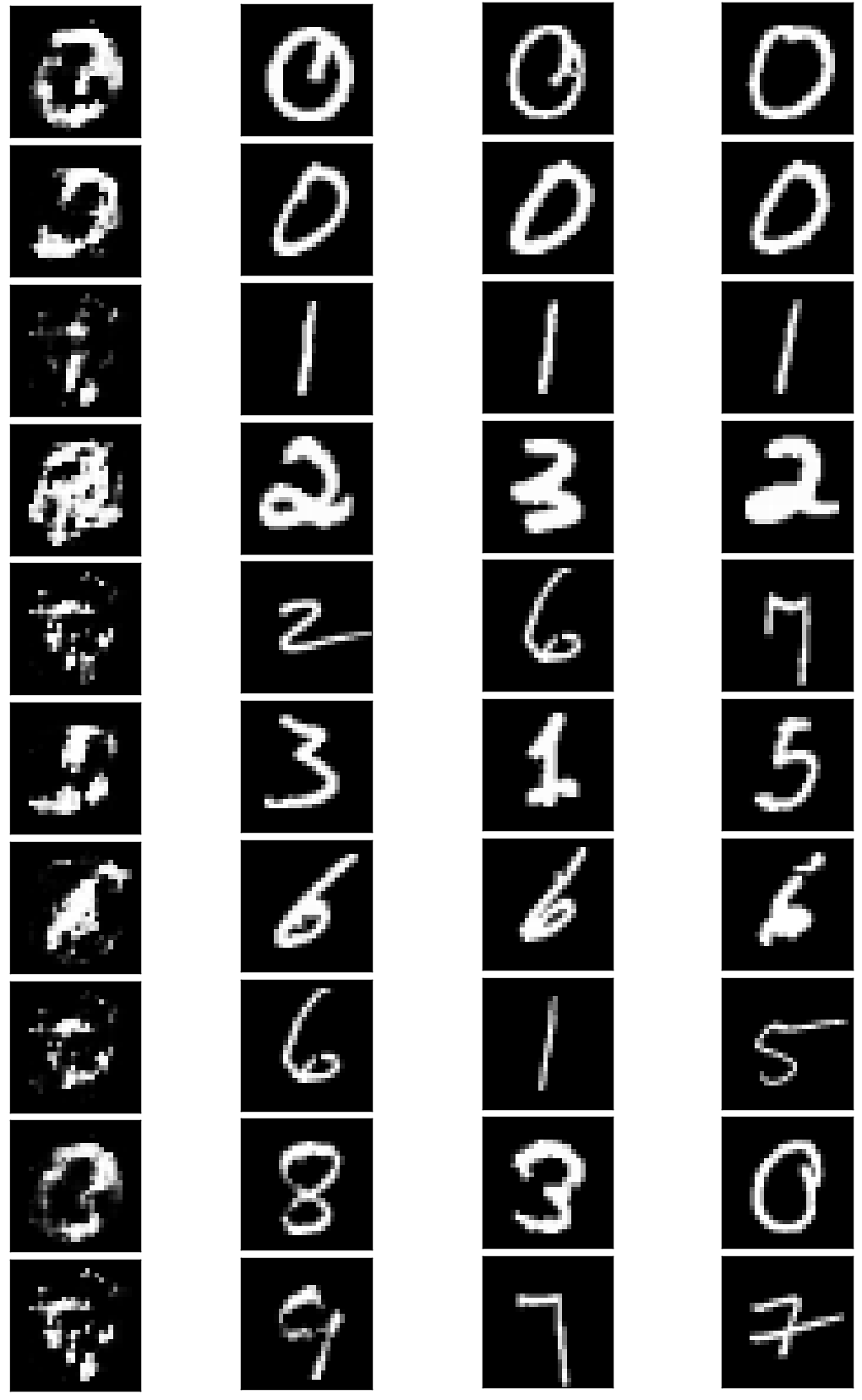} \\
(a) $\epsilon$=$\infty$ & (b) $\epsilon = 29.0$ & (c) $\epsilon = 14.0$ & (d) $\epsilon = 9.6$
\end{tabular}
\vspace{-0.1in}
\caption{Generated images with four different $\epsilon$ on MNIST dataset are plotted in leftmost column in each 
group. Three nearest neighbors of generated images are plotted to illustrate the generated data is not memorizing 
the real data and the privacy is preserved. We can see that the images get more blurred as more noise is added.}
\vspace{-0.1in}
\label{Exp1part1}
\vspace{-0.1in}
\end{figure*}

In this section, we will present extensive experiments to investigate how the
noise will affect the effectiveness of generative network on two benchmark
datasets (MNIST and MIMIC-III)\footnote{Code and experiment scripts are available at:
\url{https://github.com/illidanlab/dpgan}}. There are several notable findings
that are worth highlighting. The Wasserstein distance converges as the
training procedure goes on and exhibits fluctuation in the late stage in the
case of privacy. This fluctuation correlates well with the quality of
generated data and reflects the privacy level. In addition, our framework can
be generalized under various network structures and applied on many benchmark
datasets.

\subsection{Relationship between Privacy Level and Generation Performance}\label{subsec:1}

We conduct experiments on MNIST dataset to illustrate the relationship between
the privacy level and the quality of output images from the generator.

In this experiment, we set both the learning rate of discriminator
$\alpha_{d}$ and generator $\alpha_{g}$ to be $5.0 \times 10^{-5}$.  The
parameter clip constant $c_{p}$ is $1.0\times 10^{-2}$ such that the weights
of discriminator will be clipped back to $\left[-c_{p}, +c_{p}\right]$. We use
MNIST's training data with data size $M = 6 \times 10^{4}$ and the batch size
$m$ is set to be 64. Hence the sample probability $q$ is $\frac{m}{M} =
\frac{64}{6 \times 10^{4}}\approx 1.1 \times 10^{-3}$ . The noise scale
$\delta$ is $10^{-5}$, and the number of iterations on discriminator ($n_{d}$)
and generator ($n_{g}$) are $5$ and $5 \times 10^{5}$, respectively. Since we
use leaky ReLU as the activation function on discriminator network and ReLU on
generative network, we have $B_{\sigma^{'}} \leq 1$, where $B_{\sigma^{'}}$ is
the bound on the derivative of the activation function. Dimension of $\z$ is
100 and every coordinate is within $\left[ -1,1\right]$. We adopt similar
network structure of  DCGAN~\cite{radford2015unsupervised} with noise
generation and inference parts to protect data privacy, of which the
effectiveness has been verified in \cite{arjovsky2017wasserstein}. To impose a
certain level of noise on the network, we choose Gaussian noise with zero mean
(hence no bias) and multiple values of standard deviation. Gaussian
distribution is widely used in privacy-preserving algorithm (see Gaussian
mechanism and its variants in~\cite{dwork2013algorithmic}) and usually
results in ($\epsilon, \delta$)-differential privacy. We add
$L_{2}$-regularization on the weights of generator and discriminator, which
has little impact on our bound in Lemma~\ref{lem:gw}.

In the first experiment we investigate how the change in noise level affects
the image quality. Four groups of the generated images are plotted and shown
in Figure~\ref{Exp1part1}, corresponding to 4 different $\epsilon$ values. In
each group, the leftmost column shows the generated images for a certain
$\epsilon$ value. The rest three columns are the corresponding nearest
neighbor images from the training set,  which demonstrates that the
distortions of images are caused by noise instead of bad training images. The
distance between training images and generated images is Euclidean norm.
Comparing the generated images with their nearest neighbors, it is clear to
see that our model is not simply to memorize the training data but to be
capable of generating photographic samples with unique details. As mentioned
in~\cite{goodfellow2014generative}, these images indeed come from actual
samples of the model distributions, rather than the conditional means given
samples of hidden units. Most importantly, the generated images of each group
in Figure~\ref{Exp1part1} shows that, the larger the variance of noise is, the
blurrier the generated images would be, when all other conditions are the
same. In the sense of differential privacy, any observer who gets the
generated images can hardly know whether a data point is involved in the
training procedure or not, as elaborated in Theorem~\ref{thm:dpgan} and
illustrated by the generated images in Figure~\ref{Exp1part1}. The observer
has no way to reconstruct the training images in such case and hence the
privacy of data is protected. This demonstrates that our model successfully
addresses the privacy issue mentioned previously. The noise level ($\epsilon$)
is recommended to be tuned in a large range to guarantee good quality of
generated images. In addition, it can be seen from the results that our method
does not suffer from mode collapse or gradient vanishing, which is an
advantage that is inherited from the WGAN network structure.

\begin{figure*}[t!]
\centering
\begin{tabular}{c c c c}
\includegraphics[scale=0.18]{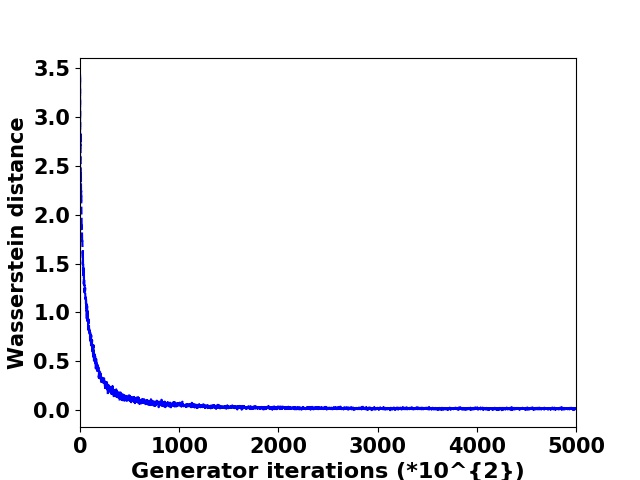} &
\includegraphics[scale=0.18]{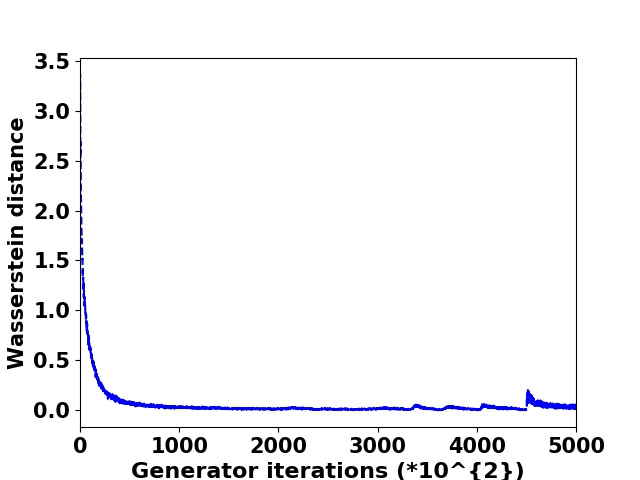} &
\includegraphics[scale=0.18]{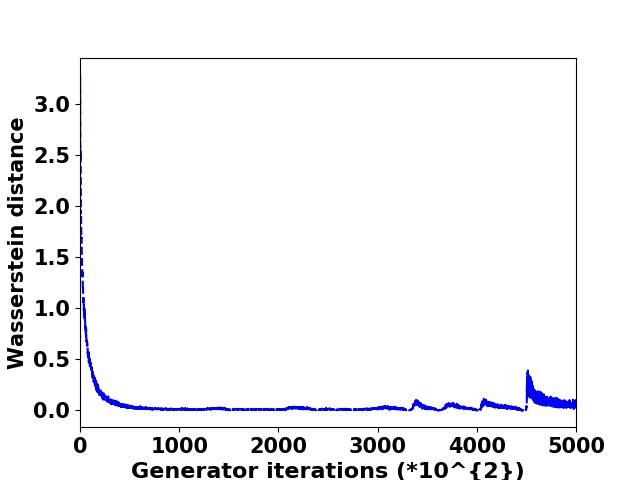} &
\includegraphics[scale=0.18]{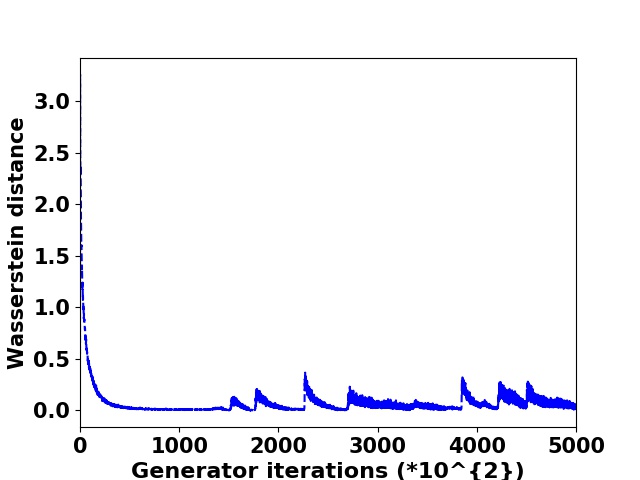} \\
(a) $\epsilon$=$\infty$ & (b) $\epsilon = 29.0$ & (c) $\epsilon = 14.0$ & (d) $\epsilon = 9.6$
\end{tabular}
 \vspace{-0.15in}
 \caption{Wasserstein distance for different privacy levels when applying DPGAN on MINST. We can see that the curves converge and exhibit more fluctuations as more noise is added.}
 \vspace{-0.05in}
 \label{fig:Exp2}
 \vspace{-0.05in}
\end{figure*}
\vspace{-0.1in}
\begin{figure*}[t!]
\vspace{-0.1in}
\centering
\begin{tabular}{c c c}
\includegraphics[scale=0.35]{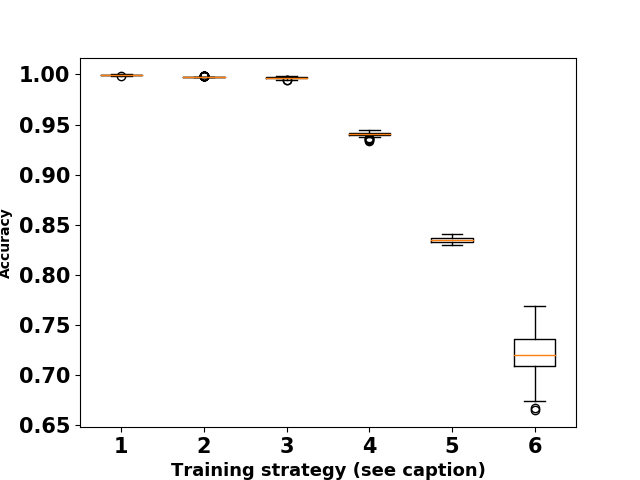} &
\includegraphics[scale=0.35]{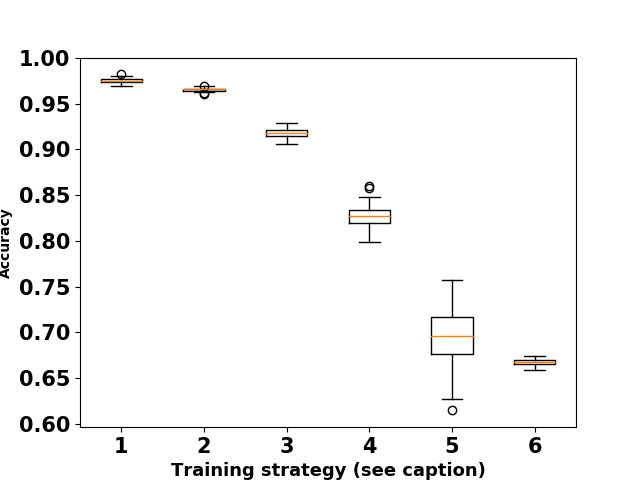} &
\includegraphics[scale=0.35]{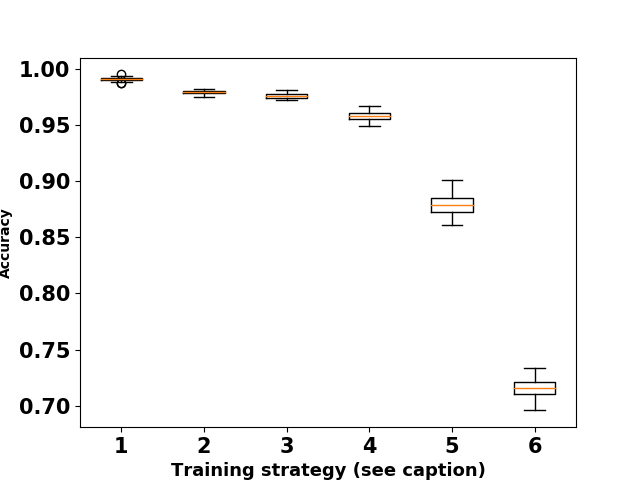} \\
(a) Digit 0 and 1 & (b) Digit 2 and 3 & (c) Digit 4 and 5
\end{tabular}
 \vspace{-0.15in}
 \caption{Binary classification task on MNIST database with different training strategies. From left to right we use training data, generated data without noise, generated data with $\epsilon=11.5,3.2,0.96,0.72$. We can see that as less noise is added, the accuracy of classifier build on generated data gets higher, which indicates that the generated data has better quality.}
 \vspace{-0.1in}
 \label{fig:subsec3}
\end{figure*}

\vspace{-0.03in}

\subsection{Relationship between Privacy Level and the Convergence of Network.}\label{subsec:2}
 
In the second experiment, we plot the Wasserstein distance for every 100
generator iterations. The result shows that the Wasserstein distance decreases
during training and converges in the end, which also correlates well with the
visual quality of the generated samples~\cite{arjovsky2017wasserstein}. The
corresponding results are shown in Figure~\ref{fig:Exp2}.  As expected, the
Wasserstein distance decreases as the training procedure goes on and
converges, which is the result of joint effect of discriminator and generator.

Despite the fluctuation caused by the min-max training itself, we can also observe
that, a smaller $\epsilon$ (hence larger noise) leads to more frequent
fluctuation and larger variance, which is especially clear in the latter half
of the curves. This conforms to the common intuition that more noise will
results in a more blurry image, which is also consistent with the results of the
previous experiment. 
One interesting phenomena is that the peaks often appear after
the convergence of Wasserstein distance. More evidences show that this might be
caused by clipping the weight. The reason is that clipping weights is
equivalent to adjusting the gradient $g_{i}$ in directions whose the corresponding gradient
$w_{i}$ magnitude is too large ($|w_{i}|>c_{p}$). 
Different from gradient
descent step (even with noise) which always changes the weight towards the
optimal solution, the effect of such adjustment is hard to predict and hence
might cause instability. This is especially clear when network converges. However,
these peaks can be quickly eliminated during the training procedure and the
network may maintain a numerical stability. This is due to the fact that the
generator is in convergence stage, which is one of the advantages of
adversarial networks. Hence our system does not suffer from divergence problem. 
Again, this experiment demonstrates the most important property of a learning system with
differential privacy consideration: there exists a trade-off between learning
performance and privacy level.

\vspace{-0.02in}

\subsection{Classification on MNIST Data}\label{subsec:3} 

In this section we conduct a binary classification task to further evaluate
the quality of the generated MNIST data. Here we use the same settings as in
subsection 1~\ref{subsec:1}. Take a pair of digits 0 and 1 as an example, we
generate 0s and 1s from their own training samples (use all samples) separately,
with different $\epsilon$ values. For each digit, we generate equal number of
data as training samples. Then for fixed $\epsilon$ (and for training set), we
randomly select 4000 samples from generated data (contains 2000 for both 0 an
1), build classifiers on them and test on MNIST's testing set. Then we repeat
this for 100 times and show the accuracy (Figure~\ref{fig:subsec3}) on testing set with
classifiers built from training data and generated ones with different
standard deviations. 
Finally we run the same procedures for digit pairs 23 and 45, as well.

The results are shown in Figure~\ref{fig:subsec3}. Despite the fact that
smaller noise makes the accuracy higher (better generated quality), the
variance of plot also decreases generally. The generate quality is little
affected below some threshold (for example, somewhere between 3.0 and 11.0 for
digit 01). Thus it is recommended to choose an $\epsilon$ larger than that
threshold (add less noise) so that the generated data will not be affected
much. Note that a threshold between 3.0 and 11.0 is quite promising privacy
level. Comparing  among three figures, digit pairs 01 performs better than the
rest two, which is due to the reason that the shapes of digit 0 and 1 make
them easy to be separated. This experiment use classification task to
demonstrate the trade-off between learning performance and privacy level.

\begin{figure*}[t!]
\centering
\begin{tabular}{c c c c}
\includegraphics[scale=0.26]{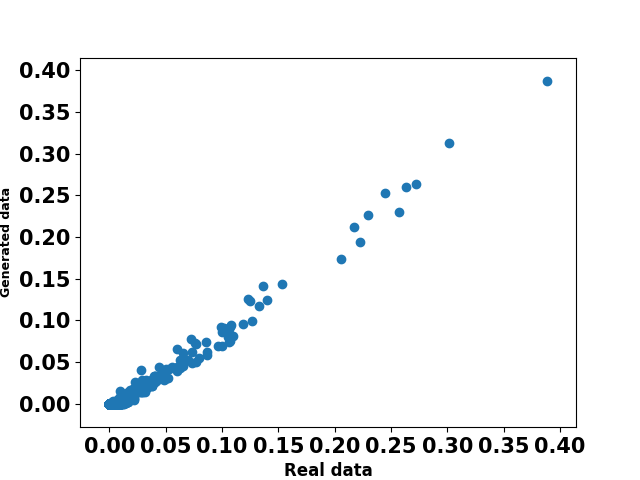} &
\includegraphics[scale=0.26]{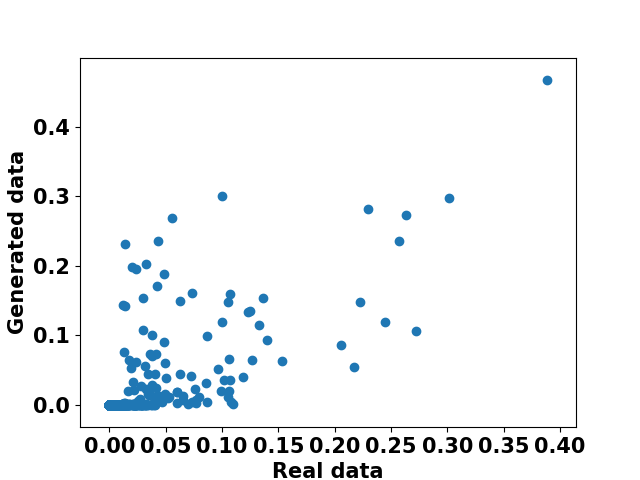} &
\includegraphics[scale=0.26]{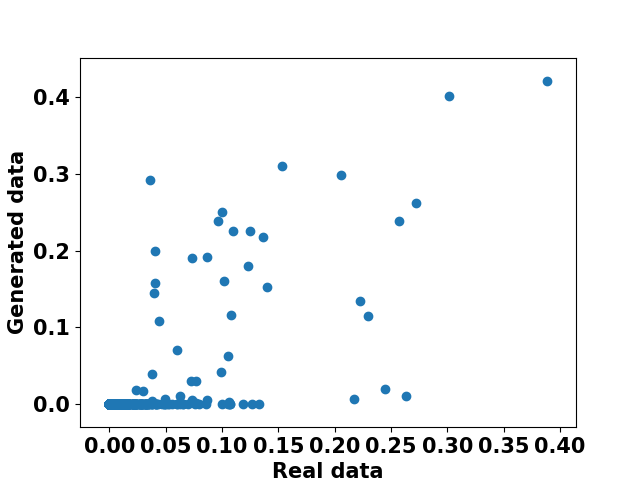} &
\includegraphics[scale=0.26]{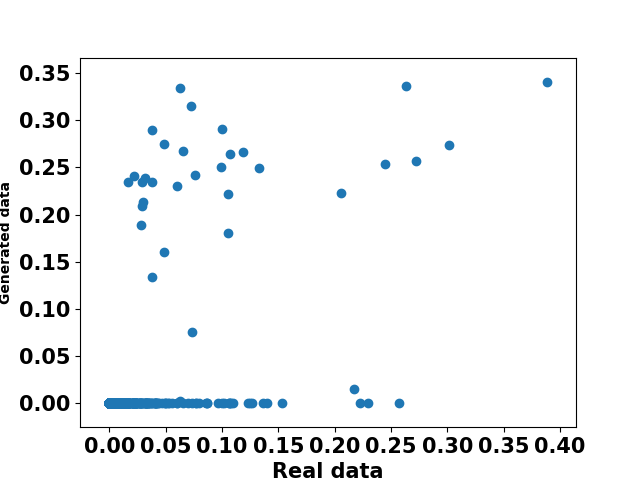} \\
(a) $\epsilon$=$\infty$ & (b) $\epsilon$=$2.31\times10^{2}$ & (c) $\epsilon$=$1.39\times10^{2}$ & (d) $\epsilon$=$96.5$
\end{tabular}
 \vspace{-0.15in}
 \caption{DWP evaluation on MIMIC-III database with different $\epsilon$ values (1070 points). We can see that as more noise is added, the distribution of generated data in each dimension becomes more deviated from the real training data.}
 \vspace{-0.1in}
 \label{fig:Exp4}			
\end{figure*}
\vspace{-0.2in}
\begin{figure*}[t!]
\centering
\begin{tabular}{c c c c}
\includegraphics[scale=0.26]{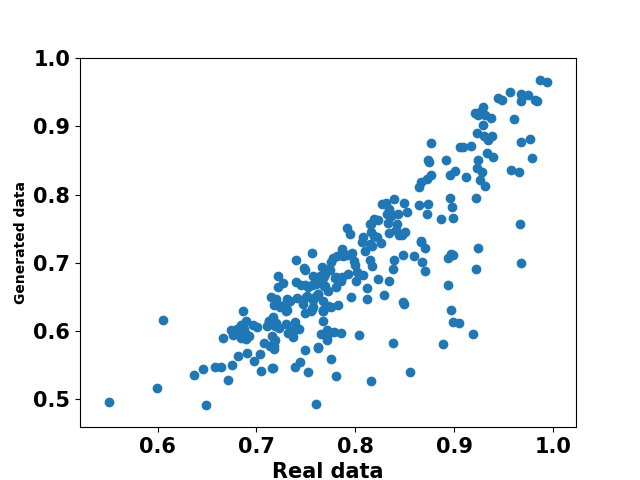} &
\includegraphics[scale=0.26]{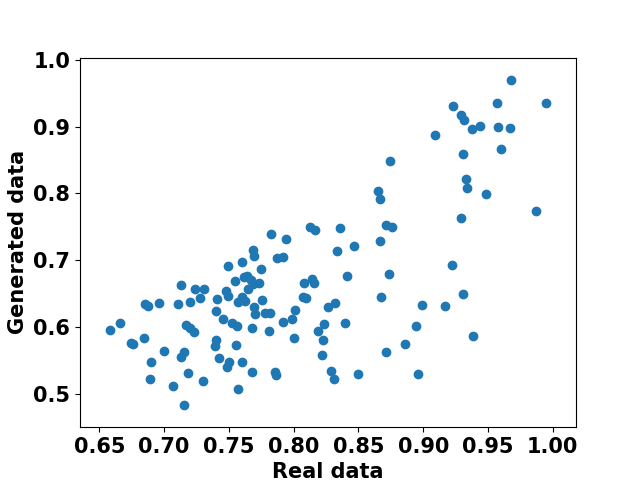} &
\includegraphics[scale=0.26]{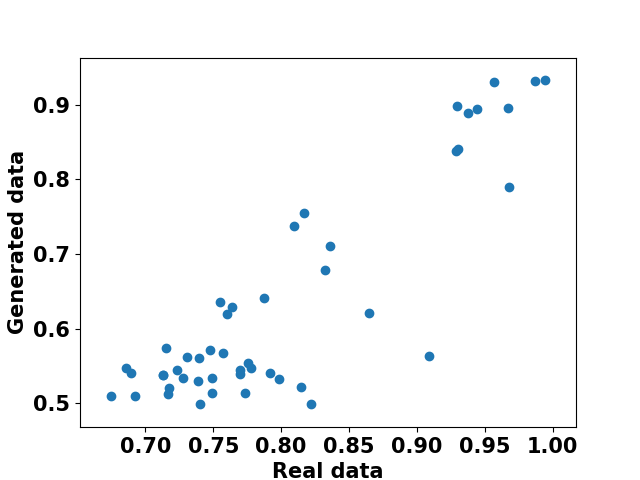} &
\includegraphics[scale=0.26]{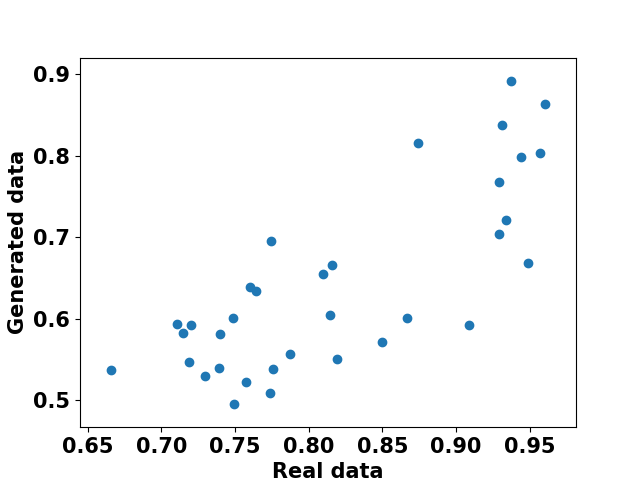} \\
(a) $\epsilon$=$\infty$ & (b) $\epsilon$=$2.31\times10^{2}$ & (c) $\epsilon$=$1.39\times10^{2}$ & (d) $\epsilon$=$96.5$
\vspace{-0.1in}
\end{tabular}
 \vspace{-0.05in}
 \caption{Dimension-wise prediction evaluation on MIMIC-III database with different $\epsilon$ values. We can see that as more noise is added, AUC value of classifier build from generated data gets lower and the data gets sparser.}
 \vspace{-0.1in}
 \label{fig:Exp5}
\end{figure*}

\vspace{0.1in}

\subsection{Electronic Health Records}\label{subsec:4}

In this section we apply DPGAN to generate Electronic Health Records (EHR)
while the privacy of patients is needed to be protected. EHR is one of the most
important information sources from which we can learn the genetics and
biological characteristics of certain population. However, the access to EHR
requires administrative permission in consideration of the privacy protection,
which is very inconvenient to the research community. 
Choi \emph{et al.} proposed
medGAN \cite{choi2017generating}, which can successfully generate EHR based on
MIMIC-III critical care
datasets~\cite{johnson2016mimic,goldberger2000components}, while the sensitive
information is not guaranteed to be protected. MIMIC-III is a well-known
public EHR database consisting of the medical records of 46,520 intensive care
unit (ICU) patients over 11 years old. In our experiments we use the extracted
ICD9 codes\footnote{International Statistical Classification of Diseases and
Related Health Problems, 9th edition} only, and group them using their first
3 digits. For each patient (1 out of 46520) in each admission to one hospital,
we record what kind of diseases this patient has and make it into a hot
vector. For example, patient A has been diagnosed with 3 diseases (with ICD9
codes 9, 42 and 146, respectively) in one admission and we use a vector to
represent the patient A's visit, where the vector has digit 1 in position 9,
42 and 146, and has digit 0 in the rest positions. Then we add up all vectors (different
admissions and different hospitals) of a certain patient and hence each patient
has one and only one vector $x \in Z_{+}^{|C|}$ with $|C| = 1071$. We then
binarize the data, where all non-zero elements are transferred to 1. These
vectors serve as summary of historical record of each patient's health
condition and can be considered as a feature for patients. Together we
can also extract useful information from these vectors. 
Notice that we remove the patient data
with missing values before feeding them into network.
 
Similar to previous experiments, we set the learning rates of both the 
discriminator $\alpha_{d}$ and generator $\alpha_{g}$
to be $5.0 \times 10^{-4}$. The parameter clip constant $c_{p}$ is $0.1$ and
$n_{d}$ is equal to 2. Also we have $m=500$, $M=46520$ for MIMIC-III dataset
and $q=\frac{500}{46520} \approx 1.1 \times 10^{-2}$. The $\delta$ is set as
$10^{-5}$. We adopt the same network structure as
in~\cite{choi2017generating}. After generating the data, we set a threshold
at 0.5 to convert the generated data matrix from continuous domain to binary domain. Since
the quality of EHR cannot be observed as images directly, we adopt the dimensional
wise probability (DWP)~\cite{choi2017generating} as a quantitative measurement
for the quality of the generated data, which is to check whether the model has
learned each dimension’s distribution correctly. Through DWP we study how the
performance of DPGAN varies with the changing of noise level.

% on how the network is affected by the noise.

The results are shown in Figure~\ref{fig:Exp4} for different noise magnitudes.
Each point in the figure is a pair of float numbers that represents Bernoulli
success probability of real data (x-axis), and generated data (y-axis) of one
dimension (corresponding to one disease). The Bernoulli success probability
(of each dimension) is the sample mean of that dimension (Maximum likelihood
estimation of independent of Bernoulli trials), which is a portion of 1 in
that column. This characterizes the rareness of that disease and hence
together reflects the distribution of diseases among population, which is a
very important statistical characteristic and can be frequently queried. Hence
there is a must to protect the people who provide this distribution by adding
noise. Despite the theoretical result in~\ref{thm:dpgan}, we can understand
the privacy protection in a intuitive way: on one hand, if no noise added
(Figure~\ref{fig:Exp4} (a)), changing database by adding one person may change
the frequency of certain disease in some extent. This change is especially
significant when the number of people in database is small or a group of
people is changed (See "group privacy" in~\cite{dwork2014algorithmic}). By
looking at this change, an observer may make some conclusions and harm the
interest of anyone who involves in the database. For example, adding a group of
people may enlarge the frequency of certain disease, if this disease is highly
related with the quality of life or it is some rare disease, health insurance
company may raise people's premiums. On the other hand, if there is
noise added (Figure~\ref{fig:Exp4} (b) to (d)), observer is not sure what is
the effect by adding this person (or these people) because the output is
uncertain (associated with a noise distribution) and the
generated data will hardly leak any patient's privacy information. 
This uncertainty gets larger when more noise is added, which can be seen from
Figure~\ref{fig:Exp4}. On the whole, it can be seen from this experiment that our
model indeed provides protection in the sense of differential privacy on
the medical data, and solves the problem we mentioned in abstract.

Note that the rareness of diseases are also well protected due to the
perturbation of noise. Assuming that there is a public-available generated EHR
data that are generated based on the EHR of a certain population, the insurance
company may raise the insurance premium for those who get rare diseases,
based on the statistical information inferred from generated EHR data. Since DPGAN
may change the rareness of diseases, the issuarance company cannot get this type of information
accurately from our generated data, thus the interest of this group of
people is guaranteed.

The results also indicate how well the generative model captures training
data's distribution. In Figure~\ref{fig:Exp4} (a), most of the points are
concentrated around line $y=x$, which indicates that our model captures each
dimension's distribution correctly. It can also be seen from
Figure~\ref{fig:Exp4} (left to right) that a large variance of noise makes more
points deviated from line $y=x$. This means that for one disease, the rareness
of generated data becomes more different from real data, which also indicates the
quality of generated data is degraded. This phenomenon matches our intuition 
that applying a higher level of noise often leads to a worse distribution approximation,
which is also consistent with evidence in Figure 2 (a)
in~\cite{choi2017generating}.

\vspace{-0.1in}
 
\subsection{Classification on EHR Data}\label{subsec:5}

Continue with previous sub-section, we use dimension-wise prediction
(DWpre)~\cite{choi2017generating} to evaluate how well the generative model
recovers the relationship among the dimensions of the data. The basic idea of
DWpre is to select the same column from training set and generated set as
target and set the rest columns as feature. Then we build logistic regression
classifiers on both of them and test on testing set. One assumption here is
that a closer performance of two classifiers indicates better quality of the
generated set. Due to the highly unbalanced testing data (0 is dominated), we
use AUC as the measurement here.

The results are shown in Figure~\ref{fig:Exp5}. Despite the fact that in most
cases, classifiers trained from real data perform better than classifiers
trained from generated data, the AUC values of generated data decrease as the
decreasing of the $\epsilon$ (more noise added). This is due to the reason
that noise perturbs the training of discriminator and affects the generator
indirectly, which leads to the deviation of output distribution from the real
one and can results in poor testing performance. It can also be seen that
there is not much decreasing in the performance, which is one of the
advantages of our model. The points get sparser as more noise is added, which
reflects another impact of noise on data. This is due to the reason that we
use logistic regression to perform binary classification, which does not allow
uni-label column. The sparse column are widely exists in original data and it
is harder for the generative model to capture the sparsity of certain column
of original data if there is more perturbation. More columns are learned as
all-zero and discarded when selected as target in classification task. In
summary, higher privacy results in less ability for generative model to
capture the inter- dimensional relationship. Also our framework successfully
addresses the issue in differential privacy system that adding noise will
cause too much decreasing in system performance.

\section{Conclusion}\label{c&f}
%!TEX root = main.tex

In this paper, we proposed a privacy preserving generative adversarial network
(DPGAN) that preserves privacy of the training data in a differentially
private sense. Our algorithm is proved rigorously to guarantee the
($\epsilon, \delta$)-differential privacy. We conducted two experiments to
show that our algorithm can generate data points with good quality and
converges under the condition of both noisy and limitation of training data,
with meaningful learning curves useful for tunning hyperparameters. For future
work we will consider reducing the privacy budget by trying different ways of
clipping, and also tighten the utility bound.

\begin{acks}
This research is supported in part by 
National Science Foundation under Grant
IIS-1565596 (JZ), IIS-1615597 (JZ), IIS-1650723 (FW) and IIS-1716432 (FW). and the Office of Naval Research under
grant number N00014-14-1-0631 (JZ) and N00014-17-1-2265 (JZ).
\end{acks}

\bibliographystyle{ACM-Reference-Format}
\bibliography{sample-bibliography-biblatex}

\end{document}